\documentclass[letterpaper, 10 pt, conference]{ieeeconf}  

\IEEEoverridecommandlockouts                              

\usepackage{Definitions}
\usepackage{color}
\usepackage{subcaption}
\usepackage{booktabs}       
\allowdisplaybreaks

\title{\LARGE \bf
Gradient Descent Finds Global Minima for Generalizable Deep Neural Networks of Practical Sizes}

\author{Kenji Kawaguchi\thanks{Kenji Kawaguchi was supported in part by NSF grants (No.1523767 and 1723381), AFOSR grant (No. FA9550-17-1-0165), ONR grant (N00014-18-1-2847), Honda Research,  and the MIT-Sensetime Alliance on AI.} \\MIT \\ Email: kawaguch@mit.edu \and Jiaoyang Huang \\ Harvard University \\ Email: jiaoyang@math.harvard.edu}

\usepackage{amsmath}

\usepackage{fancyhdr}
\begin{document}

\maketitle
\thispagestyle{empty}
\pagestyle{empty}
\thispagestyle{fancy}

\begin{abstract}
In this paper, we theoretically prove that gradient descent can find a global minimum of non-convex optimization of all layers for nonlinear deep neural networks of sizes commonly encountered in practice. The theory developed in this paper only requires the practical degrees of over-parameterization unlike previous theories. Our theory only requires the number of trainable parameters to increase linearly as the number of training samples increases. This allows the size of the deep neural networks to be consistent with practice and to be several orders of magnitude smaller than that required by the previous theories. Moreover, we prove that the linear increase of the size of the network is the optimal rate and that it cannot be improved, except by a logarithmic factor. Furthermore, deep neural networks with the trainability guarantee are shown to generalize well to unseen test samples with a natural dataset but not a random dataset.
\end{abstract}

\section{Introduction}
Deep neural networks have recently achieved significant empirical success in the fields of machine learning and its applications. Neural networks  have been theoretically studied for a long time, dating back to the days of multilayer perceptron, with focus on the  \textit{expressivity} of shallow neural networks   \cite{DBLP:journals/jc/Baum88,DBLP:journals/tc/Cover65,DBLP:journals/tnn/Huang03,DBLP:journals/tnn/HuangH91,DBLP:journals/tnn/HuangB98,yamasaki1993lower}. More recently, the  expressivity of neural networks was theoretically investigated for  modern deep architectures with rectified linear units (ReLUs) \cite{Yun2018small},  residual maps  \cite{hardt2016identity}, and/or  convolutional and max-pooling layers \cite{DBLP:conf/icml/Nguyen018}.

However, the expressivity  of a neural network does not ensure its \textit{trainability}. The expressivity  of a neural network   states that, given a  training dataset, there exists an  optimal parameter vector for  the  neural network  to interpolate that given  dataset.  It does not guarantee that an  algorithm  will be able to  find such an optimal vector, efficiently, during the training of neural networks. Indeed, 
finding  the optimal vector for a neural network  has been proven to be an NP-hard problem, in some cases \cite{bartlett1999hardness,blum1989training,livni2014computational}.

Quite recently, it was proved in a series of  papers that, if the size of a neural network is significantly larger than the size of the dataset, the (stochastic) gradient descent algorithm can find  an  optimal vector  for shallow \cite{li2018learning,du2018gradient1,song2019quadratic} and deep networks \cite{allen2018convergence,du2018gradient2,zou2018stochastic}. However, a considerable gap still exists between these trainability results and the  expressivity theories; i.e., these trainability results require a significantly larger number of  parameters, when compared to the expressivity theories. Table \ref{table:1}
summarizes the number of parameters required by each previous theory, in terms of the size $n$ of the dataset,  where  the $\tilde \Omega(\cdot)$ notation ignores the logarithmic factors and the $\poly(\cdot)$ notation hides the significantly large unknown polynomial dependencies: for example, $\poly(n)\ge n^{60}$ in \cite{allen2018convergence}.

\begin{table}[t!] 
\centering \renewcommand{\arraystretch}{0.5}
\caption{Number of parameters required to ensure the trainability, in terms of $n$, where $n$ is the number of samples in a training dataset and  $H$ is the number of hidden layers.  } \label{table:1}
\begin{tabular}{cccc}
\toprule
Reference & \# Parameters & Depth $H$ & Trainability  \\ 
\midrule
\cite{DBLP:journals/tnn/Huang03,DBLP:journals/tnn/HuangH91,DBLP:journals/tnn/HuangB98} &$\tilde \Omega(n)$ & 1,2 & No (expressivity only)\\
\midrule
\cite{hardt2016identity,DBLP:conf/icml/Nguyen018,Yun2018small}  & $\tilde \Omega(n)$ & any $H$ & No (expressivity only)\\
\midrule
\cite{li2018learning} & $\tilde\Omega(\poly(n))$ & 1 & Yes\\
\midrule
\cite{du2018gradient1} & $\tilde\Omega(n^6)$ & 1 & Yes\\
\midrule
\cite{song2019quadratic} & $\tilde\Omega(n^2)$ & 1 & Yes\\\midrule
\cite{allen2018convergence, zou2018stochastic} & $\tilde\Omega(\poly(n,H))$ & any $H$ & Yes\\
\midrule
\cite{du2018gradient2} & $\tilde\Omega(2^{O(H)}n^8)$ & any $H$ & Yes\\
\midrule
\cite{zou2019improved} & $\tilde\Omega(H^{12}n^8)$ & any $H$ & Yes\\ 
\midrule
this paper & $\tilde\Omega(n)$ & any $H$ & Yes\\
\bottomrule
\end{tabular}
\end{table}

\begin{figure}[t!]
\centering
\begin{subfigure}[b]{0.49\columnwidth}\centering
  \includegraphics[width=\columnwidth,height=0.6\columnwidth]{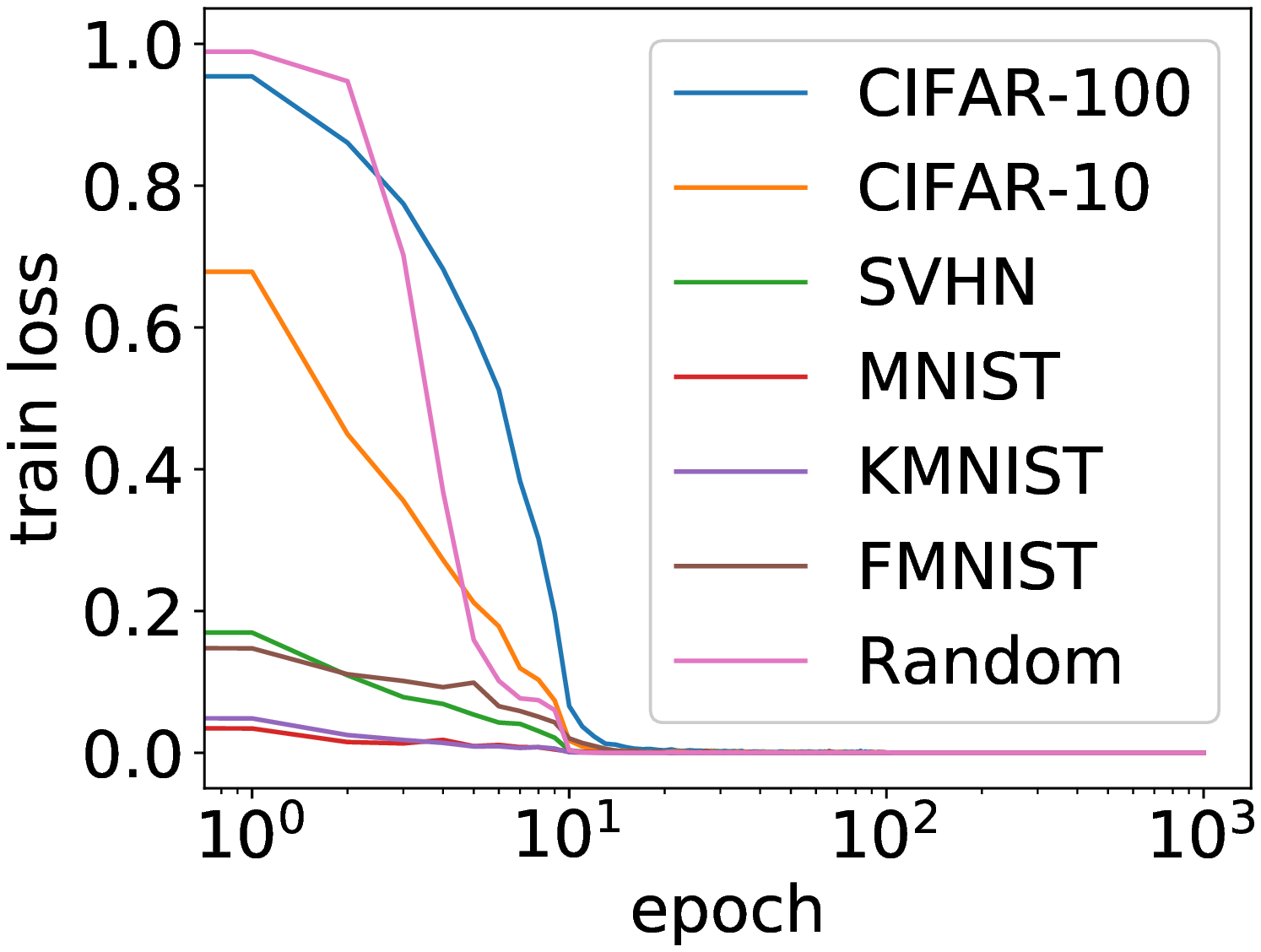}
  \caption{Training loss}  \vspace{-4pt} \label{fig:1:1}
\end{subfigure}
\begin{subfigure}[b]{0.49\columnwidth}
  \includegraphics[width=\columnwidth,height=0.6\columnwidth]{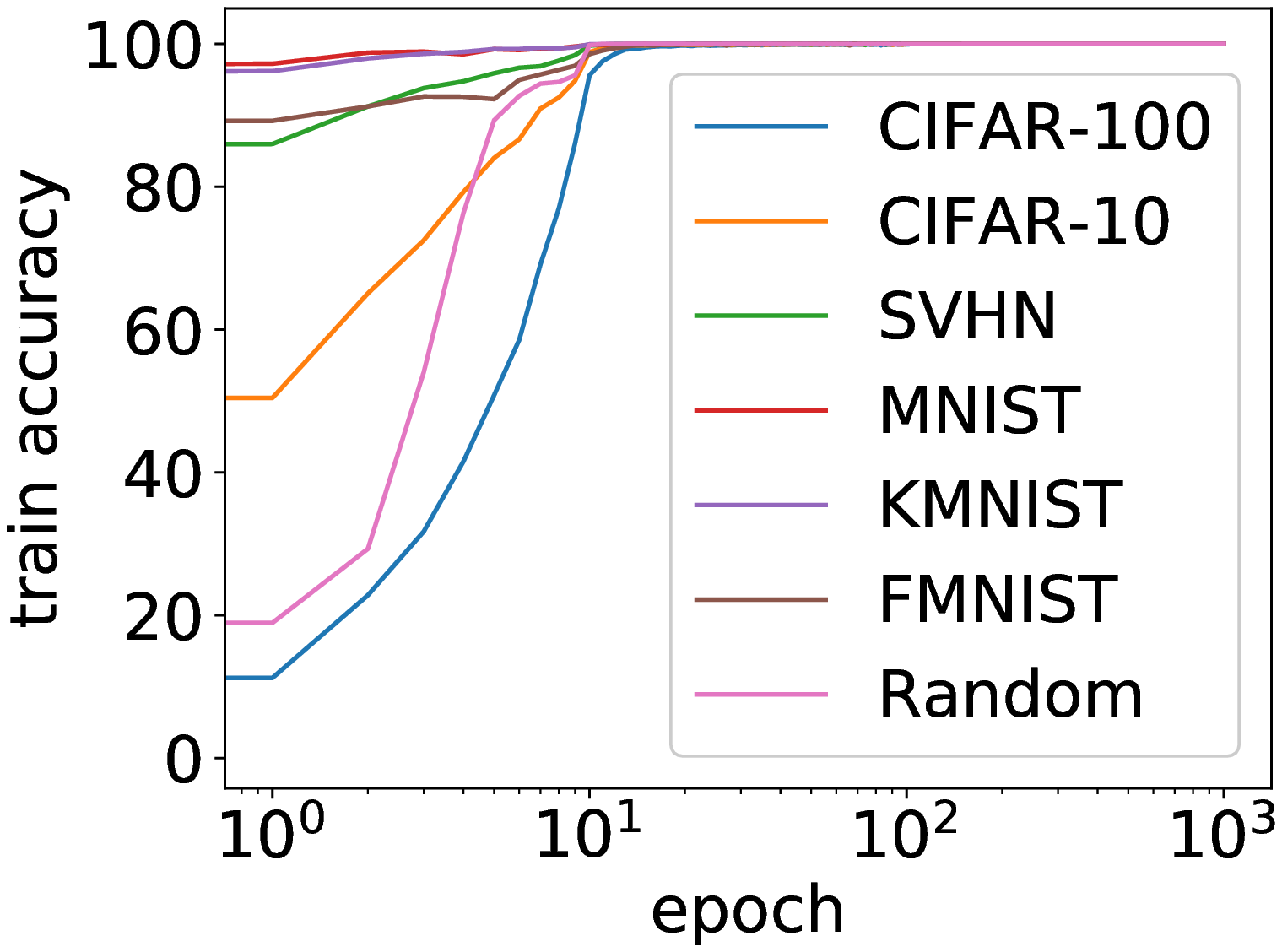}
  \caption{Training accuracy}  \vspace{-4pt}\label{fig:1:2} 
\end{subfigure}
\caption{Training loss and accuracy versus the number of epochs (in log scale) for pre-activation ResNet with  18 layers.  Training accuracy reaches 100\% (and training loss is approximately zero) for all datasets, even though the number of total parameters  is several orders of magnitude smaller than that required by the  previous theories.    } 
\label{fig:1}
\end{figure}

There is also a significant gap between the trainability theory and common practice. Typically, deep neural networks used in practical applications are trainable, and yet, much smaller than what the previous theories require to ensure trainability. 
Figure \ref{fig:1} illustrates this fact with various datasets and a pre-activation ResNet with  18 layers (PreActResNet18), which is widely used in practice. FMNIST represents the Fashion MNIST. RANDOM represents a randomly generated dataset of size 50000$ $ (with   the inputs being  $3 \times 24 \times 24$ images  of pixels drawn randomly from the standard normal distribution and the target being integer labels drawn uniformly from between 0 and 9). Here, the sizes of the training datasets vary from 50000 to 73257. For these datasets, the  previous theories require at least $n^8=(50000)^{8}$ parameters for the deep neural network to be trainable, which is several orders of magnitude larger than the number of parameters of  PreActResNet18
(11169994 parameters) or even larger networks such as WideResNet18 (36479219 parameters).

In this paper, we aim to bridge these gaps  by theoretically proving  the upper   and lower bounds for the number of  parameters required to ensure trainability. In particular, we show that  deep neural networks  with $\tilde \Omega(n)$ parameters are efficiently trainable by using a gradient descent algorithm.  That is, our theory only requires the number of total parameters to be in the  order  of $n$, which  matches the practical observations. Moreover, we demonstrate that trainable deep neural networks of size  $\tilde \Omega(n)$ are generalizable  to unseen test  points with a natural dataset, but not with  a random dataset.

\section{Preliminaries}
This paper studies feedforward neural networks with $H$ hidden layers, where $H \ge 1$ is arbitrary. Given an input vector  $x\in {\mathbb R}^{m_x}$ and a  parameter vector $\theta$, the output of the neural network is given by
\begin{equation}\label{eq:f}
f(x,\theta)=W^{(H+1)}x^{(H)}+b^{(H+1)} \in {\mathbb R}^{m_y}, 
\end{equation}
where $W^{(H+1)}\in \RR^{m_y \times m_{H}}$ and $b^{(H+1)}\in \RR^{m_y}$ are the weight matrix and bias term, respectively, of the output layer. The output of the last hidden layer $x^{(H)}$ is given by
the set of recursive equations:  $x^{(0)}=x$ and
\begin{equation}\label{e:defDNN}
x^{(l)}=\frac{1}{\sqrt{m_l}}\sigma(W^{(l)}x^{(l-1)}+b^{(l)}),\quad l=1,2,\cdots, H,
\end{equation}
where $W^{(l)}\in \RR^{m_l \times m_{l-1}}$ is the weight matrix, $b^{(l)}\in \RR^{m_l}$ is the bias term,  and $\sigma$ is the activation unit, which is applied coordinate-wise to its input. Here,   $x^{(l)}$ is the output of the $l$-th layer, which  has $m_l$ neurons. 

Then, the vector containing all trainable parameters is given by $\theta=(\vect(\Wbar^{(1)})\T, \dots, \vect(\Wbar^{(H+1)})\T)\T$, where $\Wbar^{(l)} =[W^{(l)},b^{(l)}]$ and  $\vect(M)$ represents the standard vectorization of the matrix $M$.  Thus, the total number of trainable parameters is 
\begin{equation*}
    d=\sum_{l=0}^{H}(m_l m_{l+1}+m_{l+1}),
\end{equation*}
where $m_{0}=m_x$ and $m_{H+1}=m_y$.

This paper analyzes the trainability in terms of the standard objective  of empirical risk minimization: 
$$
J(\theta) = \frac{1}{n}\sum_{i=1}^n \ell(f(x_{i},\theta),y_{i}),
$$
where $\{(x_i,y_i)\}_{i=1}^n$ is a training dataset, $y_i $ is the $i$-th target, and $\ell(\cdot,y_i)$ represents a loss criterion such as the 
squared loss or cross-entropy loss. 
The following assumptions are employed for the loss criterion $q \mapsto \ell(q,y_{i})$ and activation unit $\sigma(x)$:
\begin{assumption}\label{a:loss}
\emph{(Use of common loss criteria)} For any $i \in \{1,\dots,n\}$, the function $\ell_i(q)=\ell(q,y_{i}) \in \RR_{\ge 0}$ is differentiable and convex, and $\nabla \ell_{i}$ is $\zeta$-Lipschitz (with the metric induced by the Euclidian norm $\|\cdot\|_{2}$). 
\end{assumption}

\begin{assumption}\label{a:activation}
\emph{(Use of common activation units)}
The activation function $\sigma(x)$ is real analytic, monotonically  increasing, $1$-Lipschitz, and the limits exist as: $\lim_{x\rightarrow -\infty}\sigma(x)=\sigma_->-\infty$ and $\lim_{x\rightarrow +\infty}\sigma(x)=\sigma_+\leq+\infty$.
\end{assumption}
Assumption \ref{a:loss} is satisfied by simply using a common loss criterion such as the squared loss or cross-entropy loss.
For example, $\zeta=2$ for the squared loss, as $\|\nabla\ell_i(q)-\nabla\ell_i(q')\|_2\le 2 \|q-q'\|_2$.
The training objective function $J(\theta)$  is  nonconvex in $\theta$, even if the loss criterion $q \mapsto \ell(q,y_{i})$ is  convex in $q$. 
 
Assumption \ref{a:activation} is  satisfied by using common  activation units such as sigmoid and hyperbolic tangents. Moreover, the softplus activation,  which is defined as $\sigma_\alpha(x)=\ln(1+\exp(\alpha x))/\alpha$,  satisfies Assumption \ref{a:activation} with any hyperparameter $\alpha \in \RR_{>0}$. The softplus activation can approximate the ReLU activation  for any desired accuracy as 
$$
\sigma_{\alpha}(x) \rightarrow \mathrm{relu}(x) \text{ as } \alpha\rightarrow \infty,
$$
where $\mathrm{relu}$ represents the ReLU activation.

Throughout this paper,   neural networks are initialized with random Gaussian weights, following the common initialization schemes used in practice. More precisely, the initial parameter vector $\theta^0$  is randomly drawn as $(W^{(l)}_{ij})^0 \sim \mathcal N(0, c_w)$ and  $(b^{(l)})^0 \sim \mathcal N(0,c_b)$, where $c_w$ and $c_b$ are constants and $(W^{(l)}_{ij})^0$ and $(b^{(l)})^0$ correspond to the initial vector $\theta^0$ as $\theta^0=(\vect((\Wbar^{(1)})^{0})\T,\allowbreak \dots, \allowbreak \vect((\Wbar^{(H+1)})^0)\allowbreak\T)\T$ with $(\Wbar^{(l)})^{(0)}=[(W^{(l)})^{(0)},(b^{(l)})^{(0)}]$.
With this random initialization scheme,  the outputs are normalized properly as $\|x^{(l)}\|_2^2=O(1)$ for $0\leq l\leq H$, and $\|f(x,\theta)\|_2^2=O(m_y)$ with high probability.

\section{Main Trainability Results}

This section first introduces the formal definition of trainability, in terms of the number $d$ of parameters, and then presents our main results for the trainability.

\subsection{Problem formalization}

The goal of this section is to formalize the question of trainability in terms of the number of  parameters, $d$. Intuitively, given the dataset size $n$, depth $H$, and any $\delta>0$, we  define the \textit{probable trainability} $\Pcal_{n,H,\delta}$ as $\Pcal_{n,H,\delta}(d)=\true$ if having $d$ parameters can ensure the trainability for all datasets with probability at least $1-\delta$, and $\Pcal_{n,H,\delta}(d)=\false$ otherwise. We formalize this intuition as follows.

Let activation units $\sigma$ satisfy Assumption \ref{a:activation}. Let $\Fcal_d^H$  be the set of all   neural network architectures $f(\cdot,\cdot)$ of the  form in equation \eqref{eq:f} with    $H$ hidden layers and at most $d$ parameters. 
Let  $\Scal_{n}$  be   the set of all training datasets $S=\{(x_i,y_i)\}_{i=1}^n$ of size $n$ such that the data points are normalized as  $\|x_i\|_2^2=1$ and  $y_i\in[-1,1]^{m_y}$ for all $i \in \{1,\dots,n\}$. Let $\Lcal_S^\zeta$ be a set of all  loss functionals $L$ such that for any $L \in \Lcal_S^\zeta$, we have 
$
L(g)=\frac{1}{n}\sum_{i=1}^n \ell(g(x_{i}),y_{i})
$
and $\argmin_{g:\RR^{m_{x}}\rightarrow \RR^{m_y}} L(g)\neq \emptyset$, 
where  $g: \RR^{m_{x}}\rightarrow \RR^{m_y}$ is a function, $S \in \Scal_n$ is a training dataset, and $q \mapsto \ell(q,y_{i})$ is a loss criterion satisfying Assumption \ref{a:loss}.
For any $(\theta,\Wbar)$, we define $\psi_l(\theta,\Wbar) \in \RR^{d}$ to be the  parameter vector $\theta$ with  the corresponding $\bar W^{(l)}$ entries replaced by $\bar W$. For example, $\psi_{H+1}(\theta,\bar W)=(\vect(\Wbar^{(1)})\T, \allowbreak  \dots, \allowbreak \vect(\Wbar^{(H)})\T, \allowbreak \vect(\Wbar  )\T)\T$. We use the symbol  $\odot$ to represent the entrywise product (i.e., Hadamard product).

With these notations, we can now formalize the probable trainability $\Pcal_{n,H,\delta}$, in terms of $d$, as follows: 

\begin{definition} \label{def:prob}

 $\Pcal_{n,H,\delta}: \NN \rightarrow \{\true,\false\} $ is a function such that $\Pcal_{n,H,\delta}(d)=\true$ if and only if the following statement holds true: $\forall \zeta>0,\exists f\in \Fcal_d^H, \exists \eta \in \RR^d, \forall S \in \Scal_{n},\forall L \in \Lcal_S^\zeta$, $\exists c_{\theta} \in \RR$, and $\forall \epsilon>0$, with probability at least $1-\delta$ (over randomly drawn initial weights $\theta^{0}$), there exists $t = O(c_r \zeta/\epsilon)$ such that 
\begin{align} \label{eq:1}
J(\theta^{t})=L(f(\cdot,\theta^{t}))\le L(f^*)+\epsilon ,
\end{align}
and $\|\theta^t\|_2^2 \le c_{\theta}$,
where  $f^* \in \argmin_{g_{_{\vphantom{\&}}}:\RR^{m_{x}}\rightarrow \RR^{m_y}} L(g)$ is a global minimum of the functional $L$, $(\theta^{k})_{k\in\NN}$ is the  sequence generated by the gradient descent algorithm $\theta^{k+1}=\theta^k-\eta \odot\nabla J(\theta^{k})$, and $c_r = \max_{l\in \{1,\dots,H+1\}} \inf_{\Wbar^{*} \in \Wcal^*_l}\|\Wbar^{*} -(\Wbar^{(l)})^{0}\|^2_F$ with $\Wcal^*_l= \argmin_{\Wbar} L(f(\cdot,\psi_l(\theta^{0},\Wbar)))$.
     
\end{definition}
Here, $\Pcal_{n,H,\delta}(d)=\true$ implies  that a gradient descent algorithm finds a global minimum of all layers of a deep neural network with $d$ trainable parameters for any dataset (if a global minimum exists). To verify this, let $\tilde \Pcal_{n,H,\delta}$ be equivalent  to  $\Pcal_{n,H,\delta}$, except that inequality \eqref{eq:1} is replaced by 
\vspace{-10pt}
\begin{align} \label{eq:3}
L(f(\cdot,\theta^{t}))\le L(f(\cdot,\theta^*))+\epsilon,
\end{align}
where $\theta^* \in \RR^d$ is a global minimum of $J(\theta)=L(f(\cdot,\theta))$. As $L(f^*)\le L(f(\cdot,\theta^*))$,   $\Pcal_{n,H,\delta}(d)=\true$ implies that   $\tilde \Pcal_{n,H,\delta}(d)=\true$, which is the desired statement. 

The reason we use $\Pcal_{n,H,\delta}(d)$ instead of $\tilde \Pcal_{n,H,\delta}(d)$ is that $\tilde \Pcal_{n,H,\delta}(d)$ admits trivial and unpreferred solutions  in  that a global minimum $\theta^*$ can have a  large loss value   $L(f(\cdot,\theta^*))$ when   $f$ is restricted. As an extreme example, one can set $f$ to be a neural network with only one trainable parameter. Then, a bijection search can trivially find a global minimum with a large loss value. The use of $\Pcal_{n,H,\delta}(d)$ instead of $\tilde \Pcal_{n,H,\delta}(d)$ forces us to find nontrivial solutions with  small loss values.  

In the definition of $\Pcal_{n,H,\delta}(d)$, the network architecture $f$ and learning rate $\eta$ must be fixed for all datasets. This  forces the gradient descent algorithm to actually learn the predictor based on each  dataset, instead of encoding too much information  into the  architecture and learning rate.  

\subsection{Analysis}
The following theorem states that the probable trainability is ensured with the total parameter number $d$ being linear in $n$:

\begin{theorem} \label{thm:1}
For any $n \in \NN^+$,  $H \ge 2$, and  $\delta>0$, it holds that   $\Pcal_{n,H,\delta}(d)=\true$  for any
\begin{align*}
\scalebox{0.93}{$\displaystyle d\geq \mathfrak c \left(\left(n+m_x H^2+H^5\log \left(\frac{Hn^2}{\delta}\right)
     \right)\log \left(\frac{Hn^2}{\delta}\right) + n m_y\right)$},
\end{align*}
where  $\mathfrak c>0$ is a  universal constant. 
\end{theorem}
\begin{remark} \label{rem:1}
In Theorem \ref{thm:1}, we restrict ourselves to the case of $H\geq 2$. If $H=1$, then by setting $m_{H-1}=m_x$ and $x_i^{(H-1)}=x_i$ in the proof of Theorem \ref{thm:1}, it holds that $\Pcal_{n, 1,\delta}(d)=\true$ for any $d\geq \mathfrak cn(m_x+m_y)$. In practice, $m_x$ would be much larger than $m_y$, and if this is the case, the lower bound $\tilde\Omega(nm_x)$ for the case of $H=1$  is worse than the lower bound $\tilde\Omega(nm_y)$ in Theorem \ref{thm:1}.

\end{remark}

In other words, Theorem \ref{thm:1} and Remark \ref{rem:1} state that  there are trainable neural networks  of size $\tilde \Omega(n m_y+m_x H^2+H^5)$ if $H\ge2$, and size $\tilde \Omega(n (m_x+m_y))$ if $H=1$. This is significantly smaller than the sizes required by the  previous studies.   
 For deep neural networks, a state-of-the-art result, in terms of the size, is given in \cite{du2018gradient2}, where the neural networks are required to have size $\tilde\Omega(2^{O(H)}n^8+n^4(m_x+m_y))$.
For shallow networks, building on previous works \cite{li2018learning,du2018gradient1}, it has been proven in  \cite{song2019quadratic}  that, single-layer networks of size $\tilde\Omega(n^2(m_x+m_y))$ are trainable.
 Theorem \ref{thm:1}  proves the probable trainability for considerably smaller networks when compared with the previous results. 

Then, a natural question is  whether we can further improve Theorem \ref{thm:1} by reducing $d$ while keeping the probable trainability. The following theorem and its corollary state that Theorem \ref{thm:1} is already optimal and that it cannot be improved in terms of the order of the leading term $n m_y$: 

\begin{theorem} \label{thm:2}
There exists a universal constant $\mathfrak c>0$ such that the following holds: for any large $\beta>0$, $\frac{n m_y}{d}-1 \ge \frac{\mathfrak c \beta H\log n}{\log (1/\epsilon)}$, and deep neural network  architecture $f\in \mathcal F_d^H$, there exists a dataset $S\in \mathcal S$ such that if
\begin{align}
    \sum_{i=1}^n\|f(x_i,\theta)-y_i\|_2^2\leq \epsilon,
\end{align}
then 
$\|\theta\|^2_2 \ge n^{\beta}$.
\end{theorem}

\begin{corollary} \label{prop:1}
For any $n \in \NN^+$,  $H \ge 1$, and $\delta>0$, it holds that  $\Pcal_{n,H,\delta}(d)=\false$ for any $d<n m_y$.
\end{corollary}
Corollary \ref{prop:1} follows from Theorem \ref{thm:2} by taking the parameters $\beta= \sqrt{\log (1/\epsilon)}$ and $\epsilon\rightarrow 0$. If $n \gg H,m_x$, we have the  lower bound $d=\tilde \Omega(n m_y)$ in Theorem \ref{thm:1}, which  matches  the upper   bound $nm_y$ in Corollary  \ref{c:lengthest}, except for the logarithmic term and constant.

\section{Proofs of Trainability}

This section presents the proofs of Theorems \ref{thm:1} and \ref{thm:2}. Throughout this paper, we use $c$ and $C$ to represent various constants, which may be different from line to line.   
\subsection{Proof of Theorem \ref{thm:1}}

We first analyze the properties of the randomly initialized neural networks, and then, relate these properties to the trainability.
The following lemma shows that if the input to a layer is normalized, then the outputs and their differences of the layer  concentrate to  the corresponding means with high probability:

\begin{lemma}\label{l:length}
Consider two data points $x,x'\in \mathbb R^{m'}$  that satisfy $\|x\|_2^2=O(1)$ and $\|x'\|_2^2=O(1)$. Consider a random weight matrix $W\in \mathbb R^{m\times m'}$ with $\mathcal N(0, c_w)$ entries and a random bias term $b\in \mathbb R^{m}$ with $\mathcal N(0,c_b)$ entries. Then, the following estimates hold: \begin{align}\label{e:norm} 
& \scalebox{0.9}{$\displaystyle \mathbb P\left(\left| \frac{\|\sigma(Wx+b)\|_2^2}{m}-\mathbb E[\sigma^2(g)]\right|\geq \frac{\beta}{\sqrt{m}}\right)\leq e^{-c\beta^2}$},\\ 
\label{e:diff} 
&\nonumber \scalebox{0.9}{$\displaystyle \mathbb P  \left(\left|\frac{\|\sigma(Wx+b)-\sigma(Wx'+b)\|_2^2}{m}   - \mathbb E(\sigma(g)-\sigma(g'))^2\right|\ge \frac{\beta}{\sqrt{m}}\right)$}  \\ &\le e^{-c\beta^2},
\end{align}
where $g,g'$ are joint Gaussian variables with zero mean and covariances $\mathbb E[g^2]=c_w\|x\|_2^2+c_b$ and $\mathbb E[g'^2]=c_w\|x'\|_2^2+c_b,\mathbb E[gg']=c_w\langle x,x'\rangle +c_b$.
\end{lemma}
\begin{proof}[Proof of Lemma \ref{l:length}]
Since $W$ and $b$ have independent Gaussian entries, $(Wx)_1+b_1, (Wx)_2+b_2,\cdots, (Wx)_m+b_m$ are independent Gaussian variables with zero mean and variance $c_w\|x\|_2^2+c_b$. We can rewrite the norm as
\begin{align}
    \frac{1}{m}\|\sigma(Wx+b)\|_2^2=\frac{1}{m}\sum_{i=1}^m \sigma^2((Wx)_i+b_i).
\end{align}
By Assumption \ref{a:activation}, the activation function $\sigma$ is  $1$-Lipschitz. The random variables $\sigma^2((Wx)_i+b_i)$ are sub-exponential. Therefore, for $|\lambda|>0$ sufficiently small, we have
\begin{align}\begin{split}\label{e:expm}
    &\phantom{{}={}}\mathbb E\left[e^{\lambda (\|\sigma(Wx+b)\|_2^2-m\mathbb E[\sigma^2(g)])}\right]\\
    &=\prod_{i=1}^m\mathbb E\left[e^{\lambda (\sigma^2(W_ix+b_i)-\mathbb E[\sigma^2(g)])}\right] 
    \leq e^{cm\lambda^2}.
\end{split}\end{align}
Inequality \eqref{e:norm} follows from applying the Markov inequality to  \eqref{e:expm} and setting $\lambda=\pm \beta/(2c\sqrt{m})$. For  \eqref{e:diff}, we can rewrite the norm of the difference as
$
\frac{1}{m}\|\sigma(Wx+b)-\sigma(Wx'+b)\|_2^2 
=   \frac{1}{m} \sum_{i=1}^m \left(\sigma((Wx)_i+b_i) -  \sigma((Wx')_i+b_i)\right)^2.
$
Moreover, the random variables 
$(\sigma((Wx)_i+b_i)-\sigma((Wx')_i+b_i))^2$ are sub-exponential.
Thus, inequality \eqref{e:diff} follows from the derivation of  \eqref{e:norm}.
\end{proof}

By repeatedly applying Lemma \ref{l:length} to each layer, we obtain the following corollary, which approximates $\|x_i^{(l)}\|_2^2$ and $\|x_i^{(l)}-x_j^{(l)}\|_{2}^2$  using some constants  $p^{(l)}$ and $p_{ij}^{(l)}$ with error terms $O\left(\sum_{i=1}^{l}\frac{\beta}{\sqrt{m_i}}\right)$:   

\begin{corollary}\label{c:lengthest}
For the randomly initialized neural network, the following holds: for any $\beta>0$, with probability at least $1-O(le^{-c\beta^2})$ over $\theta^0$,
\begin{align}
    \label{e:totallen}&\|x_i^{(l)}\|_2^2=p^{(l)}+O\left(\sum_{i=1}^{l}\frac{\beta}{\sqrt{m_i}}\right),\\
    \label{e:difflen}&\|x_i^{(l)}-x_j^{(l)}\|_2^2=p_{ij}^{(l)}+O\left(\sum_{i=1}^{l}\frac{\beta}{\sqrt{m_i}}\right),
\end{align}
where    $p^{(0)}=1$,  $p_{ij}^{(0)}=\|x_i-x_j\|_2^2\geq \gamma$, and for $1\leq l\leq H$,
$p_i^{(l)}=\mathbb E[\sigma^2(g)]$, and $p_{ij}^{(l)}=\mathbb E(\sigma(g)-\sigma(g'))^2$.
Here, $g,g'$ are joint Gaussian variables with zero mean and covariances $\mathbb E[g^2]=\mathbb E[g'^2]=c_wp^{(l-1)}+c_b$ and $\mathbb E[gg']=c_w(p^{(l-1)}-p_{ij}^{(l-1)}/2) +c_b$.

\end{corollary}
\begin{proof}[Proof of Corollary \ref{c:lengthest}]
We prove the statement by induction on $l$. The statements hold trivially for $l=0$. In the following, we assume the statements for $l$, and prove them for $l+1$. From Lemma \ref{l:length}, with probability at least $1-O(e^{-c\beta^2})$, 
\begin{align}\label{e:bb1}
    &\|x_i^{(l+1)}\|_2^2=\mathbb E[\sigma^2(\tilde g)]+O\left(\frac{\beta}{\sqrt{m_{l+1}}}\right),\\
\nonumber    &\|x_i^{(l+1)}-x_j^{(l+1)}\|_2^2=\mathbb E[(\sigma(\tilde g)-\sigma(\tilde g'))^2]+O\left(\frac{\beta}{\sqrt{m_{l+1}}}\right),
\end{align}
where $\tilde g, \tilde g'$ are Gaussian variables with zero mean and covariances $\mathbb E[\tilde g^2]=c_w\|x_i^{(l)}\|^2+c_b=c_w p^{(l)}+c_b+O(\sum_{i=1}^{l}\beta/\sqrt{m_i})$,
$\mathbb E[\tilde g'^2]=c_w\|x_j^{(l)}\|^2+c_b=c_w p^{(l)}+c_b+O(\sum_{i=1}^{l}\beta/\sqrt{m_i})$, 
$\mathbb E[\tilde g\tilde g']=c_w\langle x_i^{(l)}, x_j^{(l)}\rangle+c_b=c_w (p^{(l)}-p_{ij}^{(l)}/2)+c_b+O(\sum_{i=1}^{l}\beta/\sqrt{m_i})$. 
We approximate $\tilde g, \tilde g'$ by mean zero Gaussian variables $g, g'$ such that $\mathbb E[ g^2]=\mathbb E[g'^2]=c_w p^{(l)}+c_b$, 
$\mathbb E[ g g']=c_w (p^{(l)}-p_{ij}^{(l)}/2)+c_b$. Since the activation function $\sigma$ is $1$-Lipschitz, we have
\begin{align}\begin{split}\label{e:bb2}
    \mathbb E[\sigma^2(\tilde g)]
    &=\mathbb E[\sigma^2(g)]+O\left(\sum_{i=1}^l\frac{\beta}{\sqrt{m_{i}}}\right)\\
    &=p^{(l+1)}+O\left(\sum_{i=1}^l\frac{\beta}{\sqrt{m_{i}}}\right), 
\end{split}\end{align}
\noindent and
\begin{align}\label{e:bb3}
\nonumber    \mathbb E[(\sigma(\tilde g)-\sigma(\tilde g'))^2]
    &=\mathbb E[(\sigma( g)-\sigma( g'))^2]+O\left(\sum_{i=1}^l\frac{\beta}{\sqrt{m_{i}}}\right)\\
    &=p_{ij}^{(l+1)}+O\left(\sum_{i=1}^l\frac{\beta}{\sqrt{m_{i}}}\right).
\end{align}
The statements for $l+1$ follow from combining  \eqref{e:bb1}, \eqref{e:bb2}, and \eqref{e:bb3}. \end{proof}

Now that we have an understanding of the  output $x_i^{(H)}$   \textit{for each $i$-th input}, we analyze the set of   outputs $\{x_i^{(H)}\}_{i=1}^n$ \textit{for all  inputs}. Let   $\dtil = m_H( m_{H-1}+1)$,  $\xtil =x^{(H-1)}$, $\tilde x_i=x_i^{(H-1)}$, $\wtil =[\wtil_{1}\T,\dots,\wtil_{m_H}\T]\T=\vect((W^{(H)})\T)$, and $\btil=b^{(H)}$.
Let   $M(\wtil,\btil) \in \RR^{n \times (m_H+1)}$  given by $M(\wtil,\btil)_{ij}=\sigma(\wtil_{j}\T \xtil_{i}+\btil_{j})/\sqrt{m_H}$ and $M(\wtil,\btil)_{i (m_H+1)}=1$, for $1\leq i\leq n$ and $1\leq j\leq m_H$. The following lemma shows that if $x_i^{(H-1)}$ and $x_j^{(H-1)}$ are distinguishable and the last layer is wide, then the set of   outputs $\{x_i^{(H)}\}_{i=1}^n$ is degenerate only when the weights are in a measure zero set:

\begin{lemma} \label{lemma:1}
If $\|x_i^{(H-1)}\|_2^2-\langle x_i^{(H-1)}, x_j^{(H-1)}\rangle>\mathfrak c_\gamma$ for all $i\neq j$ and $m_H\ge n$, the Lebesgue measure of the set $\{(\wtil,\btil)\in \RR^{\dtil}: \rank(M(\wtil,\btil)) <n\}$ is zero. 
\end{lemma} 
\begin{proof}[Proof of Lemma \ref{lemma:1}]
Under our assumption, the function
$$
\varphi(\wtil,\btil) = \det(M(\wtil,\btil)M(\wtil,\btil)\T)
$$
is  analytic  since $\sigma$ is analytic. With this function, we have that $\{(\wtil,\btil)\in \RR^{\dtil}: \rank(M(\wtil,\btil)) < n\}
=\{(\wtil,\btil)\in \RR^{\dtil}:\varphi(\wtil,\btil)=0\},
$
which follows the fact that since $M(\wtil,\btil) \in \RR^{n \times(m_H+1)}$, the rank of $M(\wtil,\btil)$ and the rank of the Gram matrix are equal.

Since $\varphi$ is analytic, if $\varphi$ is not identically zero ($\varphi\neq 0$),  the Lebesgue measure of its zero set 
$
\{(\wtil,\btil)\in \RR^{\dtil}:\varphi(\wtil,\btil)=0\}
$
is zero \cite{mityagin2015zero}. 
Therefore, it remains to show  that  $\varphi(\wtil,\btil)\neq 0$ for some $(\wtil,\btil)$.  

We now construct a pair $(\wtil,\btil)$ such that $M(\tilde w, \tilde b)$ is of rank $n$ and $\varphi(\tilde w, \tilde b)\neq 0$.
Set   $\wtil_{j}=\beta\xtil_{j}$ and $\btil_{j}= \mathfrak c_\gamma \beta/2 - \beta\|\xtil_{j}\|_2^2$ for $j =1,2,\dots,n$.  Then, 
 \begin{align}\begin{split} \label{eq:4}
 M(\wtil,\btil)_{ii}=\sigma(\mathfrak c_\gamma \beta/2)/\sqrt{m_H},
 \end{split}\end{align}
 and for any $j \neq i$,
 \begin{align}\begin{split} \label{eq:5}
 M(\wtil,\btil)_{ij}&=\sigma(\mathfrak c_\gamma \beta/2 +\wtil_{j}\T \xtil_{i}- \beta\|\xtil_{j}\|^2_2)/\sqrt{m_H}
\\
&\le \sigma (-\mathfrak c_\gamma \beta/2)/\sqrt{m_H},
\end{split} \end{align}
 which follows the assumption of $\|x_i^{(H-1)}\|_2^2-\langle x_i^{(H-1)}, x_j^{(H-1)}\rangle>\mathfrak c_\gamma$, and the monotonicity of $\sigma(x)$.
In \eqref{eq:4} and \eqref{eq:5}, as $\beta \rightarrow \infty$, by our Assumption \ref{a:activation},
$
M(\wtil,\btil)_{ii} \rightarrow \sigma_+/\sqrt{m_H},
$ 
and  
$
M(\wtil,\btil)_{ij} \rightarrow \sigma_-/\sqrt{m_H},
$ 
for any $j \neq i$. Therefore, for $\beta$ sufficiently large, and any $i \in \{1,\dots,n\}$,
$$
|M(\wtil,\btil)_{ii}-\sigma_-/\sqrt{m_H}| >  \sum_{j \neq i} |M(\wtil,\btil)_{ij}-\sigma_-/\sqrt{m_H}|.
$$
This means that the matrix $\Mtil = [M(\wtil,\btil)_{ij}-(\sigma_-/\sqrt{m_H})]_{1\leq i,j\leq n} \in\RR^{n \times n}$ is strictly diagonally dominant and    nonsingular; hence, its rank is $n$. This implies that    $[\Mtil, 1] \in \RR^{n \times (n+1)}$ has rank $n$, which then implies that  $[\Mtil', 1]\in \RR^{n \times (n+1)}$ has rank $n$, 
where $\Mtil'=[M(\wtil,\btil)_{ij}]_{1\leq i,j\leq n}$,  since elementary matrix operations preserve the matrix rank. Since $m_H \ge n$ and the columns of $M(\wtil,\btil)$ contain  all columns of $[\Mtil', 1]$, this implies that $\rank(M(\wtil,\btil))=n$ and   $\varphi(\wtil,\btil)\neq 0$ for this constructed $(\wtil,\btil)$, as desired. \end{proof}

We now derive an upper bound for the Lipschitz constant of the gradient of the objective function.  Let $z_{i}=[(x^{(H)}_i)\T, 1]\T$, $\Wbar=\Wbar^{(H+1)}$, and $\psi = \psi_{H+1}$.
Let $\Wbar^t=(\Wbar^{(H+1)})^t$ correspond to $\theta^t$ as $\theta^t=(\vect((\Wbar^{(1)})^{t})\T,\allowbreak \dots, \allowbreak \vect((\Wbar^{(H+1)})^t)\allowbreak\T)\T$. 
Let 
$
\bar J(w)=L(f(\cdot,\psi(\theta^0,\Wbar))),
$ 
where  $w=\vect(\Wbar) \in \RR^{\dhat}$. The following lemma bounds the Lipschitz constant.   

\begin{lemma} \label{lemma:2}
$\nabla\bar J$ is Lipschitz continuous with Lipschitz constant at most $\frac{\zeta}{n}\sum_{i=1}^n   \left\| z_{i} \right\|_2^2$. 
\end{lemma}
\begin{proof}[Proof of Lemma \ref{lemma:2}]
Let $\theta_w = \psi(\theta^0,\Wbar)$ and $\theta_{w'}=\psi(\theta^0,\Wbar')$ where $w'=\vect((\Wbar')\T)$. Then,
\begin{align*}
& \phantom{{}={}} \|\nabla \bar J(w) - \nabla \bar J(w')\|_2
\\& = \frac{1}{n} \left\| \sum_{i=1}^n \nabla_{w}(\ell_{i}\circ f_{i})(\theta_{w})-\nabla_{w'}(\ell_{i}\circ f_{i})(\theta_{w'}) \right\|_2
\\ & \le \frac{1}{n}\sum_{i=1}^n \|[z_{i }\otimes I_{m_y}]\|_{2} \left\|\nabla\ell_{i}(f_{i}(\theta_{w}))- \nabla\ell_{i}(f_{i}(\theta_{w'}))  \right\|_2
 \\ & \le \frac{1}{n}\sum_{i=1}^n  \zeta \|z_{i }\|_{2}\left\| f(x_{i},\theta_{w})- f(x_{i},\theta_{w'})  \right\|_2
  \\ & \le \scalebox{0.95}{$\displaystyle \left(\frac{\zeta}{n}\sum_{i=1}^n   \left\| z_{i} \right\|_2^2\right) \|\Wbar-\Wbar'\|_2
\le \left(\frac{\zeta}{n}\sum_{i=1}^n   \left\| z_{i} \right\|_2^2\right) \|w-w'\|_2,$}
\end{align*}
where the last line follows  $\|\Wbar-\Wbar'\|_2\le \|\Wbar-\Wbar'\|_F=\|w-w'\|_2$.  
\end{proof}

Using these lemmas, we can complete the proof of Theorem \ref{thm:1}. 
Let $f$  be an arbitrary neural network architecture satisfying   $m_1,\allowbreak  m_2,\allowbreak \dots, \allowbreak m_{H-2}\ge O( C^2H^2\allowbreak \log(Hn^2/\delta))$,  $m_{H-1}\ge O(C^2\allowbreak \log(\allowbreak Hn^2/\delta))$, and $m_H\ge O(n)$, for some  constant $C$. Since such  an arbitrary network has a total number of parameters $d = \mathfrak c(m_x H^2\log (\frac{Hn^2}{\delta}) +H^5\log^2 (\frac{Hn^2}{\delta})+n (\log (\frac{Hn^2}{\delta})+m_y))$ or higher,
all we need to show now is that such an arbitrary architecture
ensures  the desired trainability.

By setting $\beta=\sqrt{\log(Hn^2/\delta)}/c$ and $l=H-1$ in Corollary \ref{c:lengthest} and by taking a union bound, it holds that with probability at least $1-\delta$, for any $1\leq i\neq j\leq n$,
\begin{align}\begin{split} \label{eq:13}
    &\|x_i^{(H-1)}\|_2^2=p^{(H-1)}+O(1/(cC))),\\
   &\|x_i^{(H-1)}-x_j^{(H-1)}\|_2^2=p_{ij}^{(H-1)}+O(1/(cC)).
\end{split}\end{align}
In particular, it follows by considering $C$ sufficiently large that with probability at least $1-\delta$, for any $1\leq i\neq j\leq n$,
\begin{align}\label{e:xxj}
&\|x_i^{(H-1)}\|_2=O(1),\\
\nonumber &\|x_i^{(H-1)}\|_2^2-\langle x_i^{(H-1)}, x_j^{(H-1)}\rangle=(p_{ij}^{(H-1)}+o(1))/2>\mathfrak c_\gamma,
\end{align}
where the constant $\mathfrak c_\gamma$ depends only on $\gamma$.

Since the neural network is initialized by the Gaussian probability measure, which  is absolutely continuous  with respect to the Lebesgue measure,  equations \eqref{eq:13}--\eqref{e:xxj} and Lemma \ref{lemma:1} imply that, with probability at least $1-\delta$, $\rank(M(\wtil,\btil))=n$ and $\frac{1}{n}\sum_{i=1}^n   \left\| z_{i} \right\|_2 \le c_{z}$ for some constant $c_{z}$.
Accordingly, we consider the case of $\rank(M(\wtil,\btil))=n$ and $\frac{1}{n}\sum_{i=1}^n   \left\| z_{i} \right\|_2 \le c_{z}$ in the following. 

 Since $\frac{1}{n}\sum_{i=1}^n   \left\| z_{i} \right\|_2 \le c_{z}$, from Lemma \ref{lemma:2}, $\nabla \bar J$ has Lipschitz constant at most $c_z\zeta$. Therefore, for any $w',w \in \RR^{\dhat}$,
\vspace{-8pt}
\begin{align} \label{eq:7}
\nonumber  \bar J(w')  &= \scalebox{0.98}{$\displaystyle \bar J(w)+\int_{0}^1 \nabla \bar J(w+q(w'-w))\T (w'-w)dq $}
\\  & \le \scalebox{0.98}{$\displaystyle \bar J(w)+\nabla \bar J(w)\T(w' -w ) + \frac{c_z\zeta}{2} \|w'-w\|^2_2.$}
\end{align} 
We set $\eta_i=\frac{1}{c_z\zeta}$ if $i > \sum_{l=0}^{H-1}m_l m_{l+1}+m_{l+1}$ and $\eta_i=0$ otherwise. Using  \eqref{eq:7} with $w'=w^{k+1}$ and $w=w^{k}$, and  the equation of $w^{k+1}=w^{k}-\frac{1}{c_z\zeta}\nabla \bar J(w^{k})$, we obtain  
\begin{align} \label{eq:8}
\bar J(w^{k+1}) \le \bar J(w^{k})- \frac{c_z\zeta}{2} \|w^{k+1}-w^k\|^2_2\le \bar J(w^{k}) . 
\end{align}
Using  \eqref{eq:7} with $w'=w^{k+1}$ and $w=w^{k}$, we find that, for all $w\in\RR^{\dhat}$,
\vspace{0pt}
\begin{align} \label{eq:9}
& \nonumber \phantom{{}={}} \bar J(w^{k+1}) 
\\ \nonumber &\le \bar J(w^{k})+\nabla \bar  J(w^{k})\T(w^{k+1} -w^{k} ) + \frac{c_z\zeta}{2} \|w^{k+1}-w^{k}\|^2_2  
\\ \nonumber & = \scalebox{0.86}{$\displaystyle \bar J(w^{k})+\nabla \bar J(w^{k})\T(w -w^{k} ) +\frac{c_z\zeta}{2} (\|w-w^{k}\|^2_2-  \|w-w^{k+1}\|^2_2) $}
\\ & \le \bar J(w)+\frac{c_z\zeta}{2} (\|w-w^{k}\|^2_2-  \|w-w^{k+1}\|^2_2),
\end{align}
where  the third line contains only  arithmetic rearrangements  using the equation of $\nabla \bar J(w^{k})=c_z\zeta(w^{k}-w^{k+1})$, and  the last line follows the convexity of $\bar J$. 
Using  \eqref{eq:8} and \eqref{eq:9}, we have that,
for any $w\in \RR^{\dhat}$, 
\begin{align} \label{eq:10}
\nonumber t \bar J(w^{t}) &\le \sum_{k=0}^{t-1} \bar J(w^{k+1})
\\ &\le t\bar J(w^{})+\frac{c_z\zeta}{2} (\|w-w^{0}\|^2_2-\|w-w^{t}\|^2_2). 
\end{align}
Let $f^*(X) = [f^*(x_1),\dots,f^*(x_n)]\in \RR^{m_y \times n}$ and $f(X,\theta)=[f(x_1,\theta),\dots,f(x_n,\theta)]\in \RR^{m_y \times n}$. If $\rank(M(\wtil,\btil))=n$, there exists a minimum norm solution $\Wbar^{*} \in \RR^{m_y \times(m_H+1) }$ such that 
$
f(X,\psi(\theta^0,\bar W^{*})) =\Wbar^{*}M(\wtil,\btil)\T=f^*(X), 
$  
and hence $\bar J(w^*)=L(f(\cdot,\psi(\theta^0,\Wbar^{*})))=L(f^{*})$, where $w^{*}=\vect((W^*)\T)$. Thus, using  \eqref{eq:10} and recalling the parameter $c_r$ from Definition \ref{def:prob}, we have  $
\bar J(w^{t}) \le  L(f^{*})+\frac{c_{z}c_r\zeta  }{2t},
$
which implies that $
\bar J(w^{t}) \le  \bar J(w^{*})+\epsilon,
$
where $t=O(\frac{c_{r}\zeta }{\epsilon})$ and $L(f^{*})$ is the global minimum value of all layers.
 
Therefore, recalling that we have $\rank(M(\wtil,\btil))=n$ and $\frac{1}{n}\sum_{i=1}^n   \left\| z_{i} \right\|_2 \le c_{z}$ with probability at least $1-\delta$, it holds that
with probability at least $1-\delta$,
$
\bar J(w^{t}) \le  \bar J(w^{*})+\epsilon,
$
where $t=O(\frac{c_{r}\zeta}{\epsilon})$.  Then, using  \eqref{eq:10}, we have $\|w^*-w^t\| _{2}\le \|w^* - w^0\|_2$, which implies that $\|w^t\|^2_2\le (\| w^0\|_2 + 2\|w^*\|_2)^{2} \le c_{\theta}$ for some constant in $\epsilon>0$.  

\qed

\subsection{Proof of Theorem \ref{thm:2}}
We consider the following map from the parameter space to the concatenation of the output of the model at $x_1, x_2, \cdots, x_n$:
\begin{align}\label{e:map}
f_X: \theta \mapsto \vect([f(x_1, \theta), f(x_2,\theta),\cdots, f(x_n,\theta)]).
\end{align}
By Assumption \ref{a:activation}, the map $f_X$ is analytic in $\theta$. 
 We recall that the Jacobian of the map $f_X$ is defined as
$ {\rm Jac}(f_X)(\theta)=[\partial_k f(x_i,\theta)]_{1\leq i\leq n, 1\leq k\leq d }\in {\mathbb R}^{nm_y\times d}
$. In general, the image of the map $f_X$ may not be a manifold. Sard's theorem asserts that the set of critical values, i.e., the image of the set of critical points $\{\theta: \rank {\rm Jac}(f_X)(\theta)<d\},$ has Lebesgue measure $0$. 
For any noncritical point $\theta$, i.e., $\rank {\rm Jac}(f_X)(\theta)=d$, there exists a small neighborhood $U(\theta)$ of $\theta$, such that over $U(\theta)$, the rank of the Jacobian matrix of $f_X$ is $d$. Then, the rank theorem states that, the image $f_X(U(\theta))$ is a manifold of dimension $d$. Therefore, the volume of the image of the map $f_X$ is well defined, and we have the upper bound:
\vspace{0pt}
\begin{align}\label{e:image}
\nonumber\scalebox{0.98}{$\displaystyle {\rm vol} f_X(\{\theta: \|\theta\|_2^2\leq R^2\}) $}
&\scalebox{0.98}{$\displaystyle\leq {\rm vol}(\mathbb B_d(R))\sup_{\theta\in \mathbb B_d(R)} \det {\rm Jac}f_X(\theta)$}\\
&\scalebox{0.98}{$\displaystyle=\frac{\pi^{d/2}R^d}{\Gamma(d/2+1)}\sup_{\theta\in \mathbb B_d(R)} \det {\rm Jac}f_X(\theta)$},
\end{align}
\vskip-0.3cm
\noindent where $\mathbb B_d(R)$ is the radius-$R$ ball in $\mathbb R^d$.

In the following, we show that if for any point $\vect([y_1, y_2,\cdots, y_n])\in [-1,1]^{n m_y}$, there exists some $\theta\in \mathbb R^d$ with
$
\sum_{i=1}^n \|f(x_i,\theta)-y_i\|_2^2\leq \epsilon,
$ then there exists a large universal constant $\mathfrak c$ such that
$
    \frac{nm_y}{d}-1 \leq \frac{\mathfrak c \beta H\log n}{\log (1/\epsilon)}.
$
If this is the case, then the $\sqrt\epsilon$-neighborhood of the image set of the map $f_X$ covers all possible labels $[-1,1]^{n m_y}$. This fact, combined with  \eqref{e:image}, implies that 
\begin{align}\label{e:keybound}
\scalebox{0.985}{$\displaystyle\epsilon^{(n m_y-d)/2}\frac{\pi^{d/2}R^d}{\Gamma(d/2+1)}\sup_{\theta\in \mathbb B_d(R)} \det {\rm Jac}f_X(\theta)\geq 2^{n m_y}$},
 \end{align}
The following lemma provides an upper bound on $\det {\rm Jac}f_X(\theta)$, which will be used to obtain the lower bound for the Euclidean norm of $\theta$.
\begin{lemma}\label{l:Jac}
We have the following estimates for the determinant of the Jacobian of $f_X$:
\begin{align}\begin{split}\label{e:jacbound}
&\sup_{\theta\in \mathbb B_d(R)} \det {\rm Jac}f_X(\theta)\\
&\leq 
\left(\frac{2(H+1)n}{d}\left(\frac{m_y^2+H+R^2}{H+1}\right)^{H+1}\right)^{d/2}
\end{split}\end{align}
\end{lemma}
\begin{proof}[Proof of Lemma \ref{l:Jac}]
For any $\theta$, we denote the singular values of $ {\rm Jac}f_X(\theta)$ as $s_1,s_2,\cdots,s_d$. Then,
\begin{align}\begin{split}\label{e:fbound}
    \det {\rm Jac}f_X(\theta)
    =\prod_{i=1}^d s_i
    &\leq \left(\frac{\sum_{i=1}^d s_i^2}{d}\right)^{d/2}\\
    &=\left(\frac{\|{\rm Jac}f_X(\theta)\|_F^2}{d}\right)^{d/2}.
\end{split}\end{align}
In the following, we derive an upper bound for the Frobenius norm of ${\rm Jac}f_X(\theta)$. Then, inequality \eqref{e:fbound} gives an upper bound for the determinant of ${\rm Jac}f_X(\theta)$.

By the definition of the Jacobian matrix,
\vspace{-5pt}
\begin{align}\begin{split}\label{e:jac}
&\phantom{{}={}}\|{\rm Jac}f_X(\theta)\|_F^2=\sum_{i=1}^n\|\partial_\theta f(x_i,\theta)\|_F^2\\
&=\sum_{l=1}^{H+1} \sum_{i=1}^n\|\partial_{W^{(l)}}f(x_i,\theta)\|^2_F+\|\partial_{b^{(l)}}f(x_i,\theta)\|_F^2.
\end{split}\end{align}
We have the following estimates for the derivatives for $1\leq l\leq H$,
\begin{align}\begin{split}\label{e:derl}
&\phantom{{}={}}\|\partial_{W^{(l)}}f(x_i,\theta)\|^2_F+\|\partial_{b^{(l)}}f(x_i,\theta)\|_F^2\\
&\leq \|W^{(H+1)}\|_F^2(1+\|x^{(l-1)}\|_2^2)\prod_{i=l+1}^H\|W^{(i)}\|_2^2
,
\end{split}\end{align}\vskip-0.25cm
\noindent and for $l=H+1$
\begin{align}\begin{split}\label{e:derH}
&\phantom{{}={}}\|\partial_{W^{(H+1)}}f(x_i,\theta)\|^2_F+\|\partial_{b^{(H+1)}}f(x_i,\theta)\|_F^2\\
&= m_y^2(1+\|x^{(H)}\|_2^2),
\end{split}\end{align}
since the activation function is $1$-Lipschitz.
From the defining relation of a feedforward neural network, and from the fact that the activation function is $1$-Lipschitz, we obtain the following recursive bound for $x^{(l)}$,
\begin{align}\begin{split}\label{e:xbound}
    \|x^{(l)}\|_2^2
    &\leq  \|W^{(l)}x^{(l-1)}+b^{(l)}\|_2^2\\
    &\leq (\|W^{(l)}\|_F^2+\|b^{(l)}\|_F^2)(\|x^{(l-1)}\|_2^2+1).
\end{split}\end{align}
We can iterate  inequality  \eqref{e:xbound} and obtain the following bound for $1+\|x^{(l)}\|_2^2$,
\vspace{-2pt}
\begin{align}\label{e:xbound2}
\scalebox{0.95}{$\displaystyle 1+\|x^{(l)}\|_2^2\leq (1+\|x\|_2^2)\prod_{i=1}^l (1+\|W^{(i)}\|_F^2+\|b^{(i)}\|_F^2).$}
\end{align}
Using   \eqref{e:derl}, \eqref{e:derH}, and \eqref{e:xbound2}, we conclude the following estimate for the Euclidean norm of $\partial_\theta f(x_i,\theta)$,
\begin{align}\label{e:df}
&\nonumber \phantom{{}={}}\|\partial_\theta f(x_i,\theta)\|_F^2\\
&\nonumber =\sum_{l=1}^{H+1} \|\partial_{W^{(l)}}f(x_i,\theta)\|^2_F+\|\partial_{b^{(l)}}f(x_i,\theta)\|_F^2\\
&\nonumber \leq  \scalebox{0.95}{$\displaystyle 2(H+1)(m_y^2+\|W^{(H+1)}\|_F^2)\prod_{i=1}^{H} (1+\|W^{(i)}\|_F^2+\|b^{(i)}\|_F^2)$}\\
&\leq 2(H+1)\left(\frac{ m_y^2+H+\|\theta\|_2^2}{H+1}\right)^{H+1}
\end{align}
where the last line follows the AM--GM inequality. Lemma \ref{l:Jac} follows from combining \eqref{e:fbound}, \eqref{e:jac} and \eqref{e:df}, and  noticing $\theta\in \mathbb B_d(R)$.
\end{proof}

Using Lemma \ref{l:Jac}, we can  finish the proof of Theorem \ref{thm:2}. By substituting  \eqref{e:jacbound} into \eqref{e:keybound}, and raising both sides to the $1/d$-th power, we obtain the following key estimate
\begin{align}\label{e:keybound1}
    \frac{CRn}{d^{3/2}}\frac{(m_y^2+H+R^2)^{H+1}}{H^H}\geq \left(\frac{2}{\sqrt{\epsilon}}\right)^{n m_y/d-1},
\end{align}
where $C$ is a universal constant. 
It follows that there exists a large universal constant $\mathfrak c$ such that if $
    \frac{n m_y}{d}-1\geq \frac{\mathfrak c \beta H\log n}{\log (1/\epsilon)},
$ then $R\geq n^{\beta}$. This finishes the proof of Theorem \ref{thm:2}.
\qed

\section{Generalization bound and experiments}

\begin{figure}[t!]
\centering
\begin{subfigure}[b]{0.6\columnwidth}\centering
  \includegraphics[width=\columnwidth,height=0.1\columnwidth]{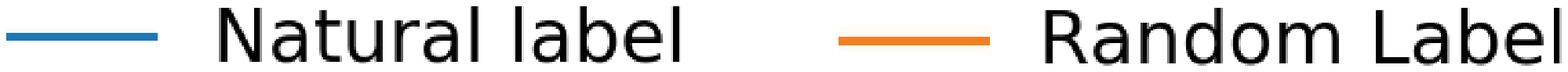}
\end{subfigure}
\begin{subfigure}[b]{0.49\columnwidth}\centering
  \includegraphics[width=\columnwidth,height=0.6\columnwidth]{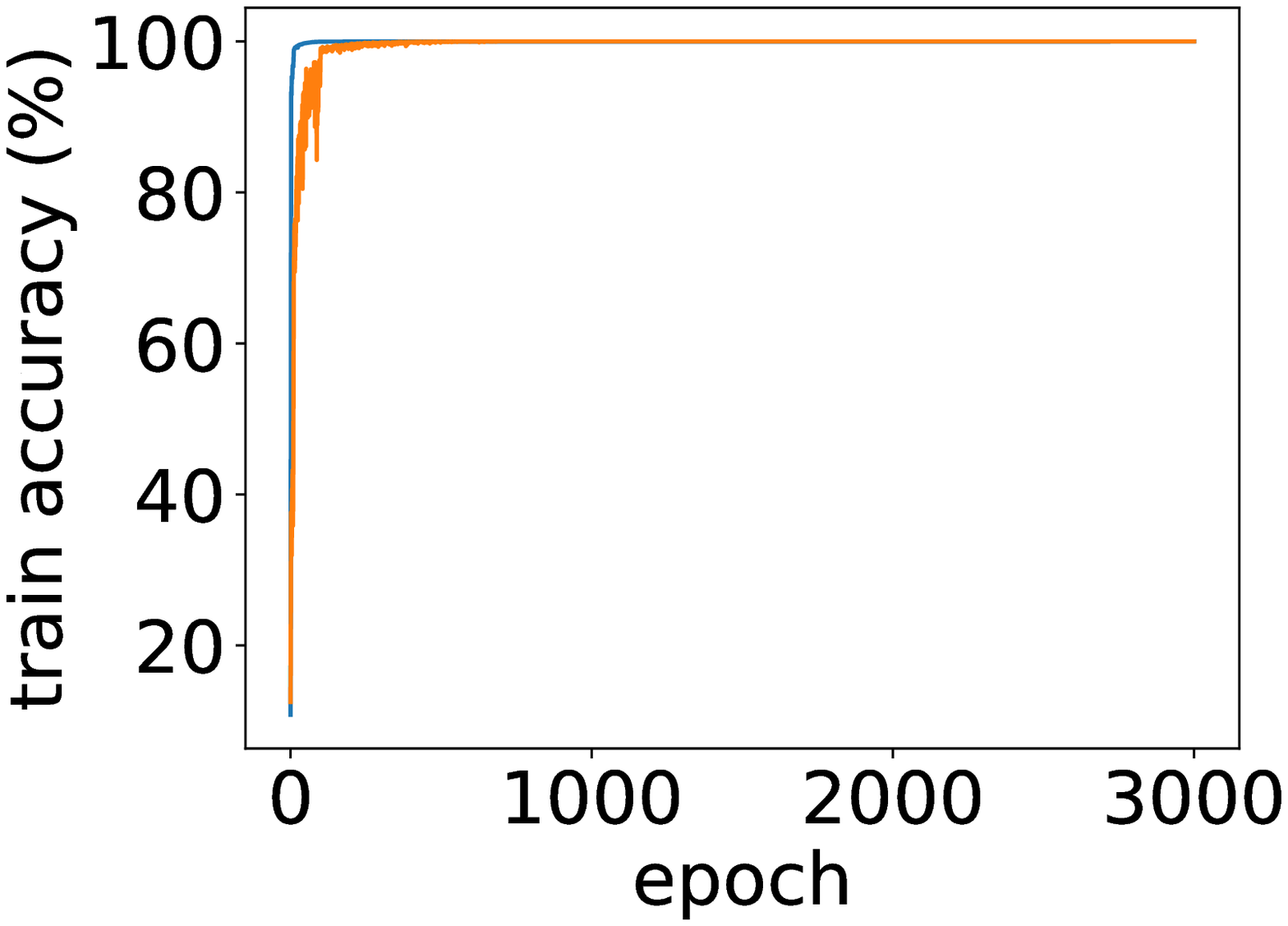}
  \vspace{-15pt}
  \caption{Training accuracy} \label{fig:2:1}
  \vspace{8pt}
\end{subfigure}
\begin{subfigure}[b]{0.49\columnwidth}
  \includegraphics[width=\columnwidth,height=0.6\columnwidth]{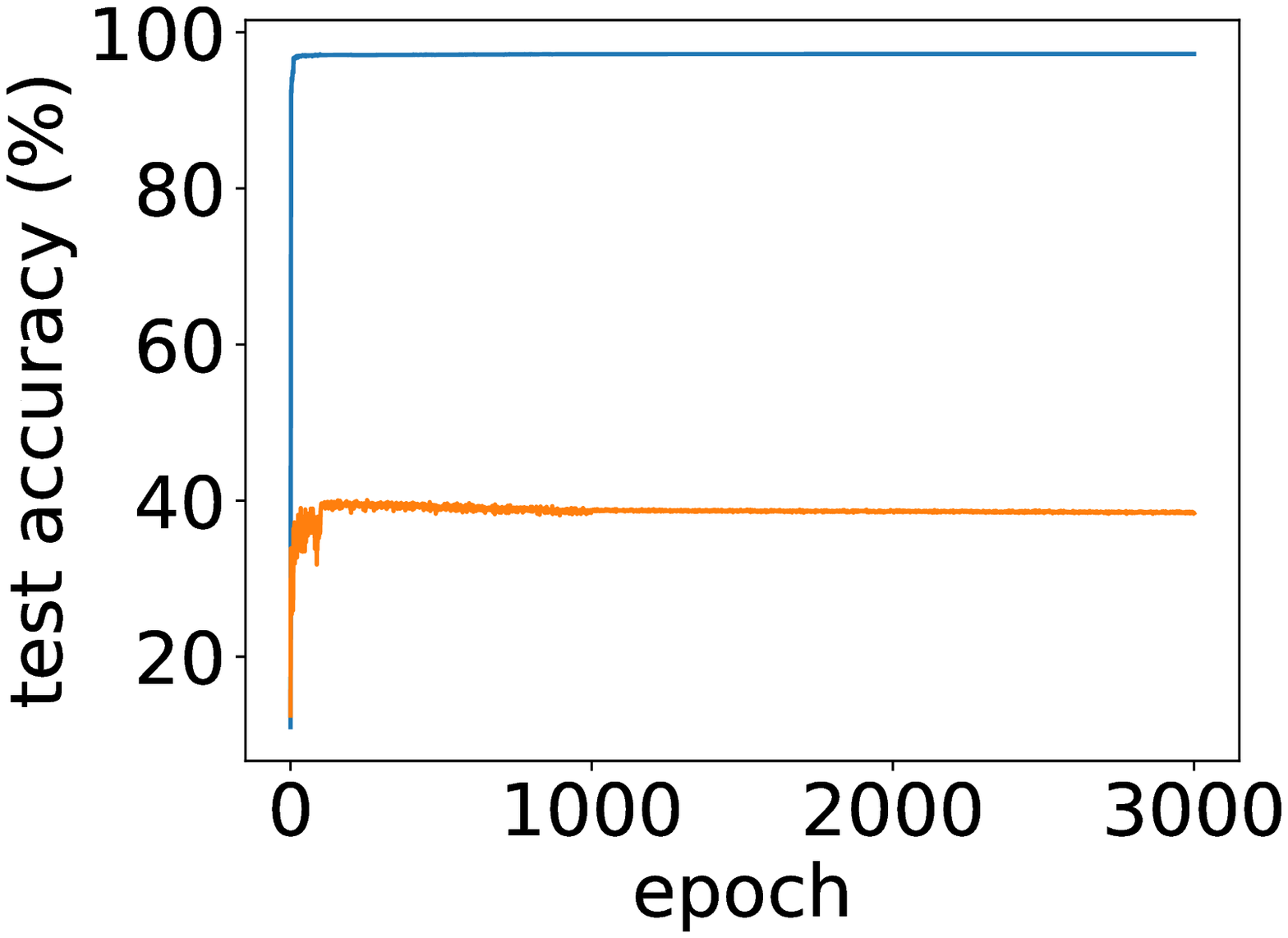}
  \vspace{-15pt}
  \caption{Test accuracy} \label{fig:2:2}
  \vspace{8pt}
\end{subfigure}
\begin{subfigure}[b]{0.49\columnwidth}\centering
  \includegraphics[width=\columnwidth,height=0.6\columnwidth]{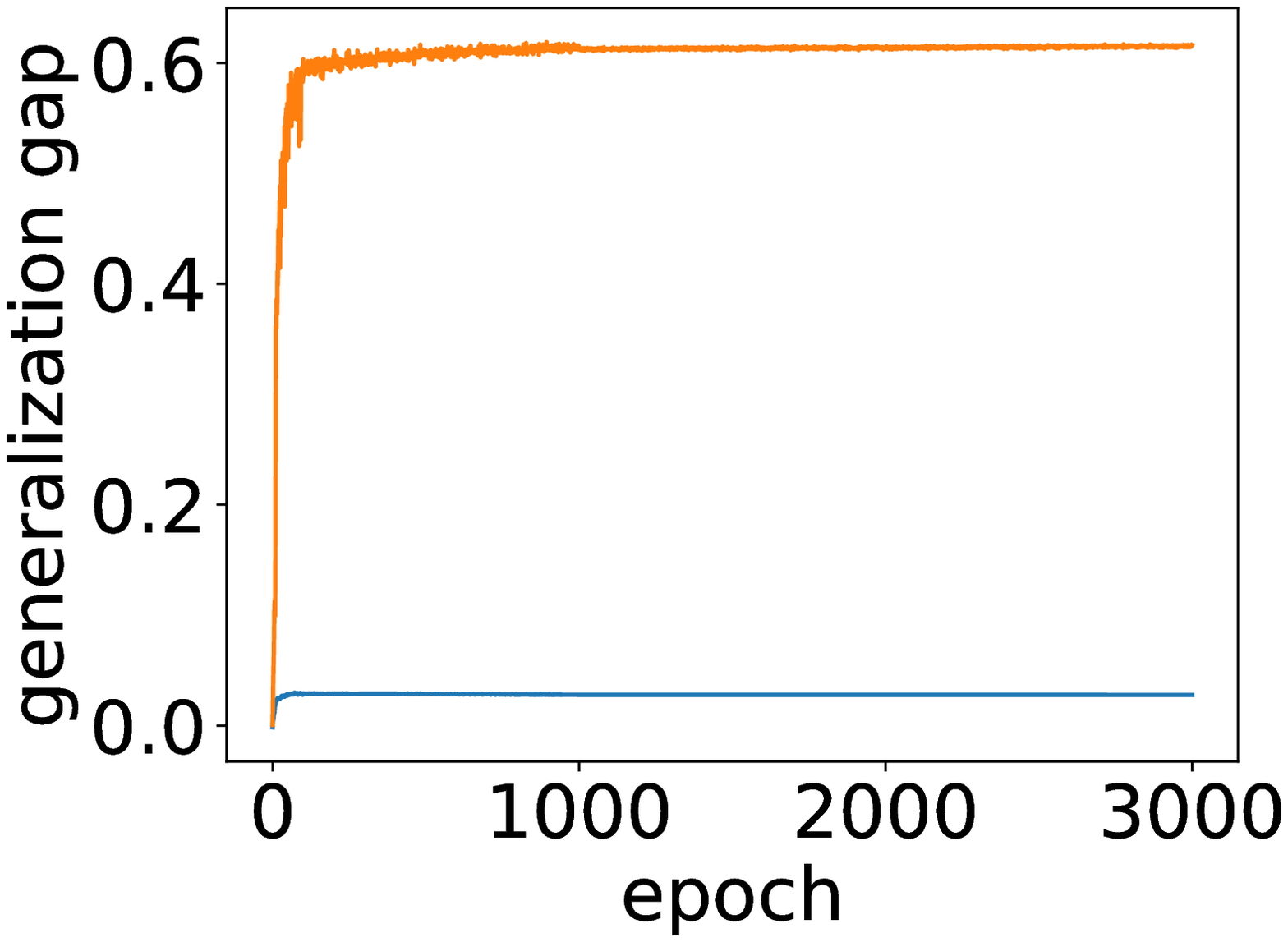}
  \vspace{-15pt}
  \caption{Generalization gap} \label{fig:2:3}
\end{subfigure}
\begin{subfigure}[b]{0.49\columnwidth}
  \includegraphics[width=\columnwidth,height=0.6\columnwidth]{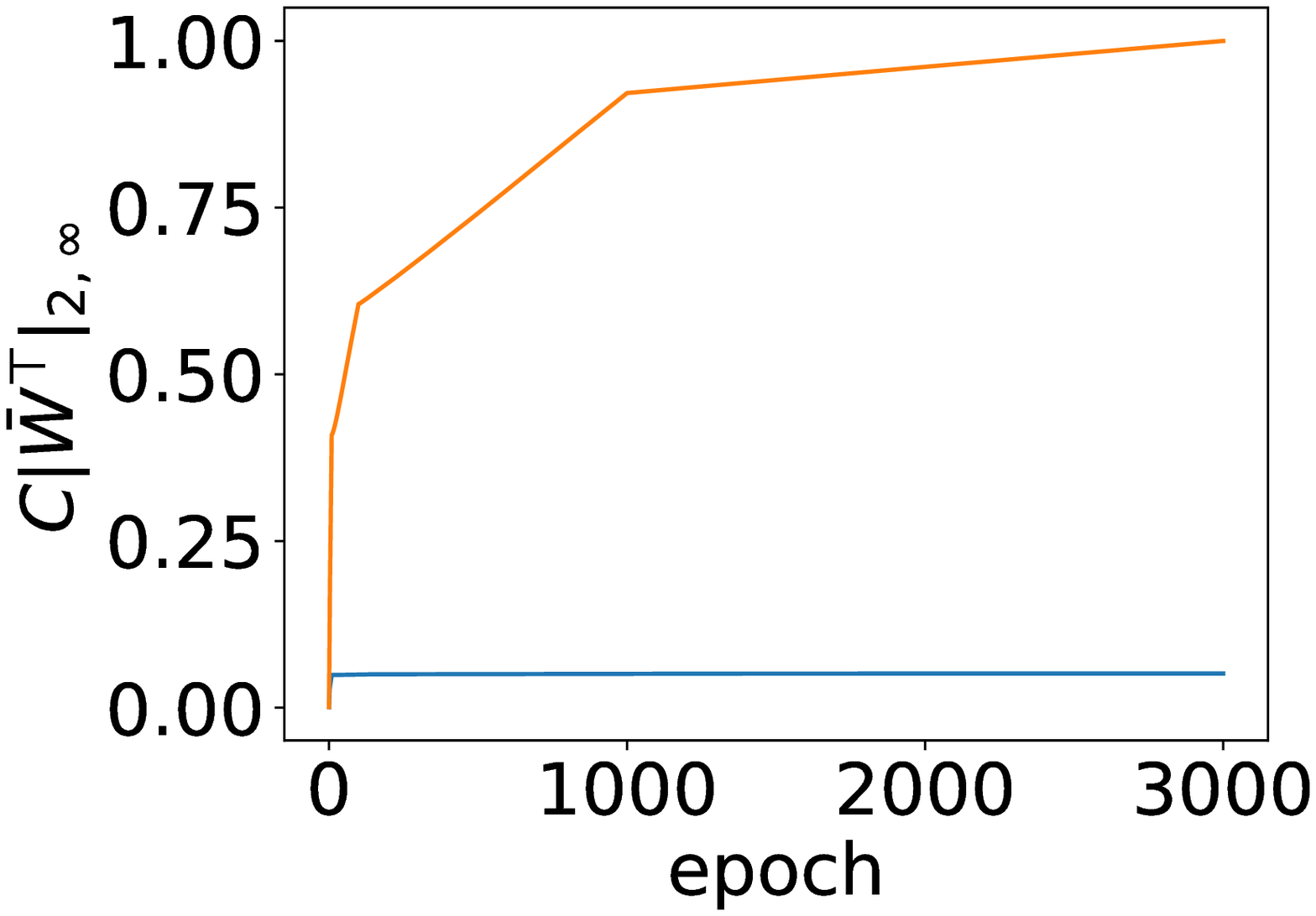}
  \vspace{-15pt}
  \caption{Weight norm} \label{fig:2:4}
\end{subfigure}
\caption{Training accuracy, test accuracy, generalization gap, and weight norm for a neural network  of practical size with the trainability guarantee, which  is constructed in the proof of Theorem \ref{thm:1}. Even though the trainable neural network has the capacity to memorize any dataset, it  generalizes well with the natural label dataset, but not with the random label dataset (where a natural label was replaced with
    a random label with probability 0.5). This behavior matches  the growth of the weight norm as predicted by Proposition \ref{prop:ge}.} 
\label{fig:2}
\end{figure}

The previous sections presented  the  construction of    deep neural network architectures of practical sizes, with the trainability guarantee. A major  question remaining now is whether the constructed    neural networks  can generalize to unseen data points after training, which is the focus of this section. 


This section considers multiclass classification with the one-hot vector $y \in \{0,1\}^{m_{y}}$. Let $j(y) \in \{1,\dots,m_y\}$ be the index of the one-hot vector $y$ having entry one as $y_{j(y)}=1$. Let $\ell_{01}$ represent the 0--1 loss as $\ell(f(x,\theta),y)= \one\{\argmax_{j} f(x,\theta)_{j} \neq j(y)\} $, with which we can write the  expected test error  $\EE_{(x,y)}[\ell_{01}(f(x,\theta),y)]$. Let $\ell_\rho$ be a standard  multiclass margin loss defined by $\ell_\rho(f(x,\theta),y)=\min(\max(1-(f(x,\theta)_{j(y)}- \max_{j' \neq j(y)} f(x,\theta)_{j'})/\rho,0),1)$.
We set  $f$ and $\eta$ as constructed in the proof of Theorem 1 (i.e.,  $m_1,\allowbreak  m_2,\allowbreak \dots, \allowbreak m_{H-2}= O(H^2\allowbreak \log(Hn^2/\delta))$,  $m_{H-1}= O(\log(\allowbreak Hn^2/\delta))$, and $m_H= O(n)$).

The following proposition provides a data-dependent generalization bound, which shows  that the  trainable deep networks can generalize to unseen data points if the   weight norm    turns out to be small after training:

\begin{proposition} \label{prop:ge}
 Fix  $\rho>0$ and $\varsigma\ge 1$. Then, for any $\delta'>0$, with probability at least $1-\delta-\delta'$ over  $\theta^0$ and i.i.d. $((x_i,y_i))_{i=1}^n$, the following holds for any $\theta^t$ generated by the gradient descent (as $\theta^{t}=\theta^{t-1}-\eta \odot\nabla J(\theta^{t-1})$):
\begin{align*}
&\EE_{(x,y)}[\ell_{01}(f(x,\theta^t),y)]- \frac{1}{n} \sum_{i=1}^n \ell_{\rho}(f(x_{i},\theta^t ),y_{i}) 
\\ &\le   \frac{cm_y^2 \ceil{\varsigma\|(\Wbar^t)\T\|_{2,\infty}}}{\rho\varsigma  \sqrt{n}}\frac{}{} + \sqrt{\frac{\ln \frac{\pi^2 \ceil{\varsigma\|(\Wbar^t)\T\|_{2,\infty} }^2}{\delta'}}{2n}}.
\end{align*}
for some constant $c=O(1)$.
\end{proposition}

Figure \ref{fig:2} shows the training accuracy, test accuracy, generalization gap, and  weight norm of one of the neural networks trained with our trainability guarantee. In the figure,  the trainable  deep neural network generalizes well with a  natural dataset, while it does not generalize well with a random dataset, as predicted by the values of the weight norm.  Here, we use the softplus activation, $H=2$, $m_{1}= 16\allowbreak \log(\allowbreak Hn^2/\delta)$, and $m_H=4n$. We employ the MNIST dataset \cite{lecun1998gradient}, which is a popular dataset for recognizing handwritten digits with  $m_x=784$ and $m_y=10$. For the random-label experiment, the natural labels in the MNIST dataset are replaced by randomly generated labels. The generalization gap plotted in subfigure \ref{fig:2:3} is the value of (training accuracy - test accuracy)/100. The weight norm plotted in subfigure \ref{fig:2:4} is the value of $C\|\Wbar\T\|_{2,\infty}$, where $C$ is the normalization  constant.

\section{Conclusions}
In this paper, we have proven that there are trainable and generalizable deep neural networks of sizes growing only linearly in the dataset size $n$. We have shown that this is already the optimal rate in terms of the dataset size $n$ and that it cannot be improved further, except by a logarithmic factor. In terms of the rate, these theoretical results are consistent with the  practical observations and  previous expressivity theories. Future work involves   improvements in terms of constant and logarithmic factors. 

Looking forward, the formalization of the probable trainability $\Pcal_{n,H,\delta}$ would contribute to set a common language in the future studies on trainability. For example, one can consider data-dependent probable trainability by redefining $\Scal_{n}$ and architecture-dependent probable trainability by reformulating  $\Fcal^H_d$, in the definition of $\Pcal_{n,H,\delta}$. Our trainability results  differ from recent results of  practical guarantees on loss landscape with representation learning effects \cite{kawaguchi2019depth,kawaguchi2019every}.    

\section*{Appendix}
\vspace{-2pt}

\subsection{Proof of Proposition \ref{prop:ge}}
\vspace{-2pt}

Define $\Theta_k=\{\theta \in \RR^{d}: (\exists \Wbar \in \Wcal_k)[\theta=\psi(\theta^{0},\Wbar)]\}$ for all $k \in \NN^+$, where $\Wcal_k=\{\Wbar\in \RR^{m_{y} \times (m_H+1)}:k-1\le \varsigma\|\Wbar\T\|_{2,\infty}< k]\}$. Let $\Tcal(\Theta_k)=\{x \mapsto f(x,\theta)_j: \theta \in \Theta_{k}, j \in \Jcal\}$ where $\Jcal=\{1,\dots,m_y\}$. Then, the previous result \cite{koltchinskii2002empirical} implies that for any $\delta'_k>0$, with probability at least $1-\delta'_k$,
the following holds for all $\theta\in \Theta_k$:
$
\EE_{(x,y)}[\ell_{01}(f(x, \theta),y)] - \frac{1}{n} \sum_{i=1}^n \ell_{\rho}(f(x_{i}, \theta),y_{i})
\le \frac{2m_y^2}{\rho}\Rfra_{n}(\Tcal(\Theta_k)) + \sqrt{\frac{\ln (1/\delta'_k)}{2n}},
$ where   $\Rfra_{n}(\Tcal(\Theta_k))$ is the  Rademacher
complexity of the set $\Tcal(\Theta_k)$, given by:
$
\Rfra_{n}(\Tcal(\Theta_k)) =\EE_{S,\xi} \left[\sup_{\theta \in \Theta, j \in\Jcal} \frac{1}{n}\sum_{i=1}^n  \xi_{i} f(x_{i},\theta)_j \right].
$   
Here,  $\xi_{1},\dots,\xi_{n}$ are independent uniform random variables taking values in $\{-1,1\}$ (i.e., Rademacher variables).  

Set $\delta'_{k}=\delta' \frac{6}{\pi^2 k^2}$, with which $\sum_{k=1}^\infty \delta'_{k}=\delta'$. By taking the union bound over $k\in \NN^+$,  for any $\delta'>0$, with probability at least $1-\delta'$, the following holds for all $k \in \NN^+$ and all $\theta\in \Theta_k$:
\vspace{-10pt}
\begin{align} \label{eq:11}
\nonumber &\EE_{(x,y)}[\ell_{01}(f(x, \theta),y)]- \frac{1}{n} \sum_{i=1}^n \ell_{\rho}(f(x_{i}, \theta),y_{i}) 
\\ &\le   \frac{2m_y^2}{\rho}\Rfra_{n}(\Tcal(\Theta_k)) + \sqrt{\frac{\ln \frac{\pi^2 \ceil{\varsigma\|\Wbar\T\|_{2,\infty} }^2}{\delta'}}{2n}}.
\end{align}
By using the Cauchy--Schwarz inequality,
$\Rfra_{n}(\Tcal(\Theta_k))
 \le \frac{\ceil{\varsigma\|\Wbar\T\|_{2,\infty}}}{\varsigma n} \EE_{S,\xi} \left[\left\| \sum_{i=1}^n  \xi_{i} z_{i} \right\|_2  \right]$. By using linearity of expectation and  Jensen's inequality (since the square root is concave in its domain), 
$
\EE_{S,\xi} [\| \sum_{i=1}^n  \xi_{i} z_{i} \|_2  ]
 \le ( \EE_{S} \sum_{i=1}^{n} \sum_{j=1}^{n} \EE_{\xi}[\xi_i \xi_j]z_{i}^\top z_{j} )^{1/2} 
 = ( \sum_{i=1}^{n} \EE_{S} [ \left\| z_{i}\right\|_2^2]  )^{1/2}
 \le (c/2) \sqrt {n},
$
where we utilize the fact that, with probability at least $1-\delta$,  $\left\| z_{i}\right\|_2 \le c/2$ for some constant $c=O(1)$, as shown in the proof of Theorem \ref{thm:1}. 
Therefore, with probability at least $1-\delta$, 
\begin{align} \label{eq:12}
\Rfra_{n}(\Tcal(\Theta_k)\le\frac{ (c/2) \ceil{\varsigma\|W\T\|_{2,\infty}}}{\varsigma  \sqrt{n}}.
\end{align}
The desired statement follows by taking the union bound for the events of  \eqref{eq:11} and \eqref{eq:12}. \qed




\addtolength{\textheight}{-12cm}  


\bibliography{all}
\bibliographystyle{IEEEtran}

\end{document}